\documentclass[12pt]{amsart}

\newcommand{\optionaldisplay}[1]{\[#1\]}

\usepackage{times}
\usepackage[margin=1in]{geometry}
\title{Autoencoding Dynamics: Topological Limitations and Capabilities}

\newcommand{\MK}[1]{[{\color{magenta}(M.D.K.) \color{cyan}\textsc{#1}}]}
\newcommand{\EDS}[1]{{\color{blue}(E.D.S.) #1}}

\renewcommand{\MK}[1]{}
\renewcommand{\EDS}[1]{}

\usepackage{amsmath}
\usepackage{amssymb}
\usepackage{xcolor}
\usepackage{mathtools}
\usepackage{thm-restate} 
\usepackage{tikz-cd}

\newcommand{\concept}[1]{\textbf{#1}}

\newcommand{\N}{\mathbb{N}}

\newcommand{\R}{\mathbb{R}}

\newcommand{\id}{\textnormal{id}}
\newcommand{\cl}[1]{\textnormal{cl}(#1)}

\newcommand{\interiortM}[1]{\textnormal{int}_{\tilde{M}}(#1)}

\newcommand{\D}{\mathcal{D}}
\newcommand{\E}{\mathcal{E}}
\newcommand{\V}{\mathcal{V}}

\newcommand{\DE}{D\circ E}
\newcommand{\DiEi}{D_i\circ E_i}

\newcommand{\sph}{\mathbb{S}}

\newtheorem{Th}{Theorem}
\newtheorem{Co}{Corollary}

\newtheorem{Prop}{Proposition}

\newtheorem*{Quest-non}{Question}

\newtheorem{Rem}{Remark}

\usepackage{marvosym}

\newcommand{\pic}[2]{\includegraphics[scale=#1]{FIGURES/#2}}
\newcommand{\picc}[2]{\begin{center}\pic{#1}{#2}\end{center}}

\usepackage{mathtools}
\usepackage{amsmath,amsfonts,amssymb,bm}
\usepackage{graphicx}
\usepackage{hyperref}
\usepackage{geometry}
\usepackage{microtype}
\usepackage{physics}
\usepackage{algorithm}
\usepackage{algpseudocode}

\begin{document}
\begin{abstract}
\noindent
Given a ``data manifold'' $M\subset \R^n$ and ``latent space'' $\R^\ell$, an autoencoder is a pair of continuous maps consisting of an ``encoder'' $E\colon \R^n\to \R^\ell$ and ``decoder'' $D\colon \R^\ell\to \R^n$ such that the  ``round trip'' map $\DE|_M$ is as close as possible to the identity map $\id_M$ on $M$.
We present various topological limitations and capabilites inherent to the search for an autoencoder, and
describe capabilities for autoencoding dynamical systems having $M$ as an invariant manifold.
\end{abstract}

\maketitle

\begin{center}
  Matthew D. Kvalheim\\
  Department of Mathematics and Statistics\\
  University of Maryland, Baltimore County, MD, USA\\
{\sc email:} kvalheim@umbc.edu\\
{\sc orcid:} 0000-0002-2662-6760

\medskip
  Eduardo D. Sontag\\
  Departments of Electrical and Computer Engineering and Bioengineering\\
  Affiliate, Departments of Mathematics and Chemical Engineering\\
  Northeastern University, Boston, MA, USA\\
{\sc email:} e.sontag@northestern.edu\
{\sc orcid:} 0000-0002-2662-6760

\end{center}

\medskip
\centerline{{\sc keywords:}  autoencoders, dynamical systems, encoding dynamics, differential geometry}

\section{Introduction}

Many natural and engineered dynamical systems evolve on manifolds of intrinsically low dimension, even when the observed data lie in extremely high-dimensional ambient spaces. For example, a simple pendulum can be fully described by its angular position and velocity---two state variables---yet a video sequence of that pendulum might consist of hundreds of thousands of pixels per frame. In such cases, the observed dynamics are governed by a low-dimensional latent structure embedded nonlinearly within the high-dimensional observation space. Learning this underlying manifold and the associated dynamical laws that govern motion on it, and representing these on an explicit manifold---for example, an Euclidean space (Fig.~\ref{fig:dynamic_auto})--- is essential for compression, prediction, control, and scientific understanding. Recent advances in representation learning---particularly autoencoders and related neural architectures---have provided powerful new tools for identifying these intrinsic coordinates directly from data, without prior knowledge of the governing equations or state variables. Notably, \cite{Chen2022} demonstrated automated discovery of fundamental variables hidden in experimental data, revealing how deep learning can uncover intrinsic dynamical coordinates even from complex visual observations such as pendulums, swinging ropes, or flames.

\begin{figure}[ht]
\picc{0.3}{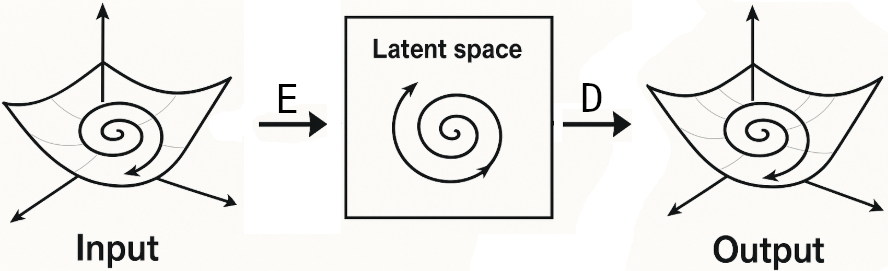}
\caption{Encoding manifold dynamics into latent Euclidean space, followed by decoding.}
\label{fig:dynamic_auto}
\end{figure}

The ability to infer low-dimensional dynamical manifolds from high-dimensional data has become central to modern data-driven modeling and system identification. Methods such as Koopman operator learning~\cite{Brunton2016,Lusch2018}, deep and variational autoencoders~\cite{Hinton2006,Kingma2014}, and physics-informed neural networks~\cite{Raissi2019} have been employed to learn latent coordinates that both compactly represent observations and evolve according to smooth, often interpretable, dynamical laws. Such representations are particularly valuable when direct measurement of state variables is infeasible---as in fluid flows, neural recordings, or robotic perception---yet one wishes to recover governing equations or perform control in reduced-order coordinates. By combining manifold learning with dynamical consistency, autoencoder-based frameworks bridge classical system identification with modern deep learning, offering a data-driven pathway to reconstruct state-space models and uncover the geometry of complex dynamical phenomena.

From a theoretical standpoint, this paper addresses a foundational question underlying these empirical advances: under what conditions can a low-dimensional dynamical system be faithfully represented through an autoencoder structure? We show that, at least for small times, given either discrete- or continuous-time dynamics $F$ on a smooth manifold \(M\) of intrinsic dimension \(m\), there exists a corresponding dynamical system $G$ on a latent Euclidean space \(\mathbb{R}^m\), together with an encoder--decoder pair that intertwines (or \emph{interlaces}) the two dynamics. That is, the encoder maps trajectories on \(M\) corresponding to a dynamics $F$ into latent trajectories obeying a well-defined Euclidean dynamics $G$, while the decoder reconstructs the original motion on \(M\). 
\begin{equation}\label{eq:intro_diagram}
\begin{tikzcd}
	\R^m 
	\arrow[dashed, "G"]{r} 
	\arrow["D"', transform canvas={xshift=-2pt}]{d} 
	& 
	\R^m 
	\arrow["D"', transform canvas={xshift=-2pt}]{d} 
	\\
	M 
	\arrow["F"]{r} 
	\arrow["E"', transform canvas={xshift=2pt}]{u} 
	& 
	M 
	\arrow["E"', transform canvas={xshift=2pt}]{u}
\end{tikzcd}
\end{equation}
Such a correspondence formalizes the intuition that autoencoders can learn intrinsic state variables whose evolution mirrors that of the true system. However, we also show that global constructions are obstructed by topological constraints---no single coordinate chart can smoothly parametrize the entire manifold. Consequently, the interlacing property can be established only on large subsets of \(M\), excluding regions of small measure where such global coordinates fail to exist. These results provide a rigorous mathematical underpinning for the representational assumptions implicit in manifold-learning and latent-dynamics models, clarifying both their power and their limitations.

This work builds directly on our previous paper, \emph{``Why should autoencoders work?''}~\cite{kvalheim2024why}, which provided a mathematical framework for understanding when autoencoders can successfully represent data lying on a low-dimensional manifold embedded in a high-dimensional space. That earlier study focused on static settings, analyzing conditions for the existence of encoder--decoder pairs that are approximately invertible and faithful to the manifold's topology and geometry. The present paper extends those ideas in two complementary directions. First, we incorporate \emph{dynamics}, asking when one can construct an autoencoder that intertwines with the evolution of a dynamical system on a manifold---thus providing latent coordinates whose dynamics mirror those of the original system. Second, we return to the \emph{static case}, but with a broader goal: to investigate, from a topological and geometric standpoint, the intrinsic capabilities and limitations of autoencoders even in the absence of dynamics, with a particular focus on what restrictions there are on the dimension of latent spaces. Together, these two lines of inquiry reveal both the promise and the fundamental obstructions inherent in learning manifold representations, clarifying the precise sense in which autoencoders can or cannot discover globally consistent latent spaces.

\section{Limitations and capabilities}\label{sec:illustrate-results}

In this paper, a smooth ($C^\infty$) \concept{manifold} $M$ has a possibly empty boundary $\partial M$. 
A manifold $M$ is \concept{closed} if $M$ is compact and $\partial M = \varnothing$.
Throughout, let $n\in \N$ and $M$ be a smoothly embedded submanifold of $\R^n$ of dimension $m$, and assume for simplicity that $M$ is connected.

We begin by formalizing an observation of Batson et al. \cite{Batson2021}.
\begin{Prop}\label{prop:easy}
The following two statements are equivalent:
\begin{itemize}
    \item 
There exist continuous (resp. smooth) maps $E\colon \R^n\to \R^\ell$, $D\colon \R^\ell\to \R^n$ satisfying $\DE|_M=\id_M$.
\item 
$M$ admits a topological (resp. smooth) embedding into $\R^\ell$.
\end{itemize}	
\end{Prop}

\begin{proof}
		If continuous  maps $E\colon \R^n\to \R^\ell$, $D\colon \R^\ell \to \R^n$ satisfying $\DE|_M = \id_M$ exist, then $E|_M\colon M\to E(M)$ is a continuous bijection with continuous inverse $D|_{E(M)}\colon E(M)\to M$, so $E|_M\colon M\to \R^\ell$ is a topological embedding.
		If moreover  $E$ and $D$ are smooth, then differentiating the expression $\DE|_M=\id_M$ yields that $E|_M$ is an immersion and hence also a smooth embedding.
		
		Conversely, if $M$ admits a topological (resp. smooth) embedding $\tilde{E}\colon M\to \R^\ell$, then any continuous (resp. smooth) extensions $E\colon \R^n\to \R^\ell$ of $\tilde{E}$ and $D\colon \R^\ell\to \R^n$ of $E|_M^{-1}\colon E(M) \to M \subset \R^n$ satisfy $\DE|_M=\id_M$.	
\end{proof}

In typical applied usage of autoencoders, $M$ does not admit an embedding into $\R^\ell$, so ``ideal'' autoencoders producing lossless encoding/decoding do not generally exist. 
And yet, autoencoders are empirically useful. 
This begs the question of whether ``approximately ideal'' autoencoders still exist.

\subsection{Limitations}\label{subsec:limitations}

Our first theorem gives a particularly strong ``no'' answer when the latent space dimension $\ell$ is strictly less than the data manifold dimension $m$. 
\begin{Th}\label{th:no-dim-decrease}
	Assume that  $m > \ell$.
	 Then for any relatively open subset $U\subset M$, there is $C > 0$ such that, for any continuous $E\colon \R^n\to \R^\ell$ and $D\colon \R^\ell\to \R^n$, 
	\begin{equation*}
	\sup_{x\in U}	\|\DE(x)-x\|\geq C.
	\end{equation*}
\end{Th}

\begin{proof}
Fix any smooth embedding $\varphi\colon \sph^\ell \hookrightarrow U$. (Note that such an embedding exists because of the strict inequality $m > \ell$.)
By continuity and compactness of $\sph^\ell$, there is $C > 0$ such that
\begin{equation}\label{eq:min-antip-dist}
	\min_{y\in \sph^\ell}\|\varphi(y)-\varphi(-y)\| = 2C.
\end{equation}
The Borsuk-Ulam theorem furnishes a point $z\in \sph^\ell$ such that $E\circ \varphi(z)=E\circ \varphi(-z)$ \cite[Cor.~2B.7]{hatcher2002algebraic}.
This, \eqref{eq:min-antip-dist}, and the triangle inequality imply that
\begin{align*}
	\|\DE(\varphi(z))-\varphi(z)\| + \|\DE(\varphi(-z))-\varphi(-z)\|  &\geq 2C,
\end{align*}
so at least one of the terms on the left side of the latter inequality is $\geq C$, completing the proof.
\end{proof}

For the case that $m = \ell$, we have the following result (in which $C$ can always be taken $\geq$ the \emph{reach} of $M$ \cite[p.~ 9]{kvalheim2024why}). 

\begin{Th}[{\cite[Thm~2]{kvalheim2024why}}]\label{th:reach-intro}
Assume that $m = \ell$ and $M$ is a closed manifold. 
There is $C > 0$ such that, for any pair of continuous maps $E\colon \R^n\to \R^\ell$, $D\colon \R^\ell\to \R^n$, 
	\begin{equation*}
	\max_{x\in M}\|\DE(x) - x \|\geq C.
\end{equation*}
\end{Th}

Observe that, in contrast to the case $m>\ell$, Theorem~\ref{th:reach-intro} does not rule out the existence of (relatively) open subsets $U$ for which a perfect round-trip ($\DE(x) = x$ for all $x\in U$) is possible. Indeed, locally this can trivially be accomplished on Euclidean charts. In fact, we next describe how such perfect round-trips can be done not merely locally but also ``almost'' globally.

\subsection{Capabilities of static autoencoders}\label{subsec:capabilities_static}
In this subsection and the next one we exclusively consider the case that the dimensions $m$ of $M$ and $\ell$ of $\R^\ell$ are equal.

While Theorem~\ref{th:reach-intro} shows that there can be global obstructions to autoencoding when $m=\ell$, there are no local obstructions.
In fact, there are no local obstructions in a strong sense: ideal autoencoding is always possible on arbitrarily ``big'' open subsets of $M$ with nice properties.
We defer to section~\ref{sec:proofs} the proof of the following theorem, which is a stronger version of \cite[Thm~1]{kvalheim2024why}.
We say that a subset of a smooth manifold $M$ is \concept{full measure} if it is the complement of a set of ``measure zero'', which is well-defined for any such $M$ \cite[p.~128]{lee2013smooth}.

\begin{restatable}[]{Th}{ThmPos}\label{th:positive}
Assume that $m = \ell$ and $M$ is compact.
For any finite set $S\subset M$ there are relatively open subsets
\optionaldisplay{U_1 \subset \cl{U_1}\subset U_2 \subset \cl{U_2}\subset \cdots}
of $M$ such that:
\begin{itemize}
	\item  $\bigcup_{i\in \N} U_i$ is full measure in $M$ and $\bigcup_{i\in \N} U_i \cap \partial M$ is full measure in $\partial M$,
	\item  $U_i$ and $\cl{U_i}$ are contractible sets containing $S$ for each $i \in \N$, and
	\item there are smooth maps $E_i\colon \R^n\to \R^\ell$, $D_i\colon \R^\ell\to \R^n$ satisfying $\DiEi|_{U_i} = \id_{U_i}$ for each $i$.	
\end{itemize}
Moreover, if $\partial M \!\!=\!\! \varnothing$, then $\cl{U_i}$ is a smooth manifold diffeomorphic to a closed $m$-disc for each $i$.
\end{restatable}

In a standard way, $M \subset \R^n$ inherits a finite Borel measure generalizing length and surface area to higher dimensions, which we call the \concept{intrinsic measure} \cite[p.~23]{kvalheim2024why}.
Consider any finite Borel measure $\mu$ that is absolutely continuous with respect to the intrinsic measure on $M$. It follows from Theorem~\ref{th:positive} that the measure of the complement $M\setminus U_i$ approaches zero as $i \rightarrow +\infty$. This is a consequence of the following elementary measure-theoretic fact, applied to the complements $F_i:=M\setminus U_i$: if $F_i$ is a decreasing sequence of measurable sets of finite measure, and denoting $F:=\lim_{i\rightarrow\infty} F_i$, then $\lim_{i\rightarrow\infty} \mu(F_i) = \mu(F)$. In our case, $F$ is the complement of $\bigcup_{i\in \N} U_i$, hence $\mu(F)=0$. 
A similar result holds for the measure of $\partial M\setminus (U_i \cap \partial M)$. 

The following corollary concerning $L^p$-norms was suggested to us by Dr. Joshua Batson  and follows readily from the preceding theorem (see \cite[Rem.~8]{kvalheim2024why} for a proof for $p=2$; the same proof works for general $p$).
It implies in particular that, under mild assumptions, the standard $L^2$ training error used to compute autoencoders can always be made arbitrarily small when $m=\ell$.

\begin{Co}\label{co:lp-intro}
Assume that $m = \ell$ and $M$ is compact.
Let $\mu$ be any finite Borel measure that is absolutely continuous with respect to the intrinsic measure on $M$.
For any $\varepsilon, p > 0$, there are smooth maps $E\colon \R^n\to \R^\ell$, $D\colon \R^\ell\to \R^n$  such that
\begin{align*}
	\int_M \|D(E(x))-x\|^p\, d\mu(x) &< \varepsilon,\\
 	\int_{\partial M} \|D(E(x))-x\|^p\, d\partial \mu(x) &< \varepsilon.
\end{align*}
\end{Co}

\subsection{Capabilities of dynamic autoencoders}\label{subsec:capabilities_dynamic}

In this section, we assume given a dynamics (discrete or continuous time) on the data manifold $M$. Without loss of generality, when $M$ is topologically closed, we may equally well (extend if necessary) start from a dynamics on the ambient space $\R^n$ for which the submanifold $M$ is invariant. 

We begin with simple corollaries of Theorem~\ref{th:positive} for autoencoding dynamical systems for (possibly) \emph{small} time.
The first corollary is for continuous-time systems and the second is for discrete-time.
For the first statement, let $C^{0,1}(\R^\ell,\R^\ell)\subset C(\R^\ell,\R^\ell)$ denote the subspace of locally Lipschitz vector fields (with subspace topology induced by the compact-open topology). 

\begin{Co}\label{co:vf-autoencoding-small-time}
	Assume that $m = \ell$ and $M$ is compact.
	Let $U_1,U_2,\ldots$ be as in Theorem~\ref{th:positive}.
        Then for each $i\in \N$ there  are smooth maps
\optionaldisplay{E\colon \R^n\to \R^\ell, \; D\colon \R^\ell\to \R^n}
satisfying $\DE|_{U_i} = \id_{U_i}$ and the following property: if $f$ is a locally Lipschitz vector field on $\R^n$ with local flow $\Phi_f$ satisfying
    $f(p)\in T_p M$ for all $p\in M$,  there is a complete locally Lipschitz vector field $g$ on $\R^\ell$ with flow $\Phi_g$ satisfying
	\begin{equation}\label{eq:vf-autoencoding-small-time-1}
		D\circ \Phi_g^t\circ E(x) = \Phi_f^t(x)
	\end{equation}
	for all $x\in U_i$ and $t$ in the connected component $I_x\subset \R$ containing $0$ of the set $\{s\in \R\colon \Phi^s(x)\in \cl{U_i}\}$.
    Moreover, $g$ is $C^k$ if $f$ is $C^k$ with $k\in \N_{\geq 1}\cup \{\infty\}$ .
	
	In particular, suppose such an $f$ is given.
    Then for any subsets 
\optionaldisplay{\E\subset C(\R^n,\R^\ell), \;\D\subset C(\R^\ell,\R^n),\;\V \subset C^{0,1}(\R^\ell, \R^\ell)\subset C(\R^\ell,\R^\ell)}
 that are dense in the compact-open topologies and any $\varepsilon, T > 0$, there are $\tilde{E}\in \E$, $\tilde{D}\in \D$, $\tilde{g}\in \V$ such that
	\begin{equation}\label{eq:vf-autoencoding-small-time-2}
	\|\tilde{D}\circ \Phi_{\tilde{g}}^t \circ \tilde{E}(x) - \Phi_f^t(x)\|< \varepsilon
	\end{equation} 
	for all $x\in U_i$ and $t\in I_x \cap [-T,T]$.
\end{Co}

\begin{proof}
Let $j\coloneqq i+1$ and let $D_{j}$, $E_j$ be as in Theorem~\ref{th:positive}.
Given $f$, let $g$ be any compactly supported and locally Lipschitz (or $C^k$ if $f$ is $C^k$) extension to $\R^\ell$ of the pushforward
\begin{align*}
(E_j|_{\cl{U_i}})_*f 
\coloneqq d E_j \circ f \circ E_j|_{\cl{U_i}}^{-1}. 	
\end{align*}
Then since $E_j|_{\cl{U_i}}^{-1}=D_j|_{E_j(\cl{U_i})}$, the chain rule and Picard-Lindel\"{o}f theorem imply that \eqref{eq:vf-autoencoding-small-time-1} is satisfied by $E\coloneqq E_j$ and $D\coloneqq D_j$  for all $x\in U_i$ and $t\in I_x$.
Finally, \eqref{eq:vf-autoencoding-small-time-2} follows from \eqref{eq:vf-autoencoding-small-time-1} and the facts that $\Phi_f$ is continuous, $(t,y,g)\mapsto \Phi_g^t(y)$ is continuous, 
\optionaldisplay{\bigcup_{y\in \cl{U_i}}(I_y\cap[-T,T])\times \{y\} \subset \R \times \R^n}
is compact for any $T>0$, and composition of continuous functions is continuous with respect to the compact-open topologies \cite[p.~64, Ex.~10(a)]{hirsch1994differential}.
\end{proof}

For discrete-time dynamical systems defined by a map $F\colon \R^n\to \R^n$ satisfying $F(M)\subset M$, there does not seem to be a directly analogous version (without imposing extra assumptions) of the ``small time'' Corollary~\ref{co:vf-autoencoding-small-time},  due to the fact that $F(U_i)\not \subset U_i$ may occur.  
The following statement reflects this.

\begin{Co}\label{co:dt-autoencoding-small-time}
	Assume that $m = \ell$ and $M$ is compact.
	Let $U_1,U_2,\ldots$ be as in Theorem~\ref{th:positive}.
        Then for each $i\in \N$ there  are smooth maps $E\colon \R^n\to \R^\ell$, $D\colon \R^\ell\to \R^n$ satisfying $\DE|_{U_i} = \id_{U_i}$ and the following property: if $F\colon \R^n\to \R^n$ is a $C^k$ map satisfying $F(M)\subset M$ with $k\in \N_{\geq 0}\cup \{\infty\}$, there is a  $C^k$ map $G\colon \R^\ell \to \R^\ell$ such that
	\begin{equation}\label{eq:dt-autoencoding-small-time-1}
		D\circ G \circ E(x) =F(x)
	\end{equation}
	for all $x\in U_i \cap F^{-1}(U_i)$.
	
	In particular, suppose such an $F$ is given.
        Then for any subsets
\optionaldisplay{E\subset C(\R^n,\R^\ell), \;\D\subset C(\R^\ell,\R^n), \;\mathcal{G} \subset C(\R^\ell,\R^\ell)}
that are dense in the compact-open topologies and any $\varepsilon > 0$, there are $\tilde{E}\in \E$, $\tilde{D}\in \D$, $\tilde{G}\in \mathcal{G}$ such that
	\begin{equation}\label{eq:dt-autoencoding-small-time-2}
		\|\tilde{D}\circ \tilde{G} \circ \tilde{E}(x) - F(x)\|< \varepsilon
	\end{equation} 
	for all $x\in U_i\cap  F^{-1}(U_i)$.
\end{Co}

\begin{Rem}\label{rem:dt-small-time-big-good-set}
	If $F$ restricts to a diffeomorphism of $M$, then $F|_M^{-1}$ sends measure zero sets to measure zero sets.
    Thus, since Theorem~\ref{th:positive} implies that $M\setminus \bigcup_{i\in \N} U_i$ is measure zero in $M$, so is  $M\setminus \bigcup_{i\in \N} F^{-1}(U_i)$.
    Since $U_i$ and $F^{-1}(U_i)$ are increasing with $i\in \N$, it follows that the complement of the set on which \eqref{eq:dt-autoencoding-small-time-1}, \eqref{eq:dt-autoencoding-small-time-2} hold can be made to have arbitrarily small measure with respect to any given continuous Riemannian metric on $M$ if $F|_M$ is a diffeomorphism $M\to M$.
\end{Rem}

\begin{Rem}\label{rem:extend-co:dt-autoencoding-small-time}
In Corollary~\ref{co:dt-autoencoding-small-time}, the conclusion \eqref{eq:dt-autoencoding-small-time-1} can be extended by replacing $G$, $F$ with $G^{\circ n}$, $F^{\circ n}$ for all $1\leq n \leq N$ for some fixed arbitrary $N\in \N$, at the cost of replacing $U_i \cap F^{-1}(U_i)$ in the following line  with $U_i \cap F^{-1}(U_i) \cap (F^{\circ 2})^{-1}(U_i) \cap \cdots \cap (F^{\circ N})^{-1}(U_i)$, and similarly for the conclusion \eqref{eq:dt-autoencoding-small-time-2}. 
Modifying Remark~\ref{rem:dt-small-time-big-good-set} accordingly then yields the same qualitative conclusion for arbitrarily large but finite numbers of iterations of the maps $F$ and $G$. 
\end{Rem}

\begin{proof}
	Let $j\coloneqq i+1$ and let $D\coloneqq D_{j}$, $E\coloneqq E_j$ be as in Theorem~\ref{th:positive}.
	Given $F$, define $V_i\coloneqq U_i \cap F^{-1}(U_i)$, $W_i\coloneqq E(V_i)$, and $G\colon \R^\ell \to \R^\ell$ to be any smooth extension of $ E_j\circ F\circ D_j|_{\cl{W_i}}\colon \cl{W_i} \to \R^\ell$.
	By Theorem~\ref{th:positive}, 
\begin{align*}
		D\circ G\circ E|_{V_i} 
&
        =  D\circ  E_j\circ F\circ D_j|_{\cl{W_i}}\circ E|_{V_i}
\\
&
        = E|_{U_i}^{-1}\circ E|_{U_i}\circ F\circ E|_{V_i}^{-1}\circ E|_{V_i}
\\
&
        = F,
\end{align*}
	which is \eqref{eq:dt-autoencoding-small-time-1}. 	
	Finally, \eqref{eq:dt-autoencoding-small-time-2} follows from \eqref{eq:dt-autoencoding-small-time-1} and the fact that composition of continuous functions is continuous with respect to the compact-open topologies \cite[p.~64, Ex.~10(a)]{hirsch1994differential}.
\end{proof}

On the other hand, the following consequences of Theorem~\ref{th:positive} are results for autoencoding continuous-time dynamical systems for \emph{large} times, but at the cost of the extra topological assumption that $M$ is a closed manifold having a smooth Euclidean covering space.
The  restrictiveness of this assumption is discussed in the remark below (recall we are assuming that $M$ is connected).
\begin{Rem}
With the exceptions of $\sph^2$ and $\R P^2$, any closed manifold $M$ of dimension $m\leq 2$ has a smooth Euclidean covering space.
But in general,  for a manifold $M$ to have a Euclidean  covering space, it is necessary that $M$ be \concept{aspherical}, meaning that all higher homotopy groups of $M$ vanish: $\pi_k(M)=\{*\}$ for all $k > 1$.
This condition is not sufficient, however, as there exist  aspherical closed manifolds not having a Euclidean covering space \cite{belegradek2015open}.
On the other hand, one sufficient condition for an aspherical closed smooth manifold $M$ to have a smooth Euclidean covering space is that $m\neq 3,4$ and its fundamental group $\pi_1(M)$ contains a finitely generated non-trivial abelian subgroup; this is \cite[Thm~1]{lee1975manifolds} combined with the fact that $\R^m$ has no exotic smooth structures for $m\neq 4$.
Another sufficient condition for a closed smooth manifold $M$ to have a smooth Euclidean covering space is that $M$ admit a Riemannian metric of nonpositive sectional curvature, by the Cartan-Hadamard theorem \cite[Thm~V.4.1]{sakai1996riemannian}. 
\end{Rem}

The next theorem is for continuous-time systems and the subsequent theorem is for discrete time.

\begin{Th}\label{th:vf-autoencoding-large-time-cover}
	Assume that $m = \ell$ and $M$ is a closed manifold for which there exists a smooth covering map $D\colon \R^\ell \to M$.
	Let $U_1,U_2,\ldots$ be as in Theorem~\ref{th:positive}.
	Let $f$ be a locally Lipschitz vector field on $M$ with flow $\Phi_f$.
	Then there is a complete locally Lipschitz vector field $g$ on $\R^\ell$ with flow $\Phi_g$ such that, for any $i\in \N$, there is a smooth map $E\colon \R^n\to \R^\ell$ (independent of $f$ and $g$) satisfying  $D\circ E|_{U_i} = \id_{U_i}$ and
	\begin{equation}\label{eq:vf-autoencoding-large-time-1}
		D\circ \Phi_g^t\circ E(x) = \Phi_f^t(x)
	\end{equation}
	for all $x\in U_i$ and $t\in \R$.
	Moreover, $g$ is $C^k$ if $f$ is $C^k$ with $k\in \N_{\geq 1}\cup \{\infty\}$.
	
	In particular, for any subsets
\optionaldisplay{\E\subset C(\R^n,\R^\ell), \;\D\subset C(\R^\ell,\R^n), \;\V \subset C^{0,1}(\R^\ell, \R^\ell)\subset C(\R^\ell,\R^\ell)}
that are dense in the compact-open topologies and any given $\varepsilon, T > 0$, there are $\tilde{E}\in \E$, $\tilde{D}\in \D$, $\tilde{g}\in \V$ such that
	\begin{equation}\label{eq:vf-autoencoding-large-time-2}
		\|\tilde{D}\circ \Phi_g^t \circ \tilde{E}(x) - \Phi_f^t(x)\|< \varepsilon
	\end{equation} 
	for all $x\in U_i$ and $t\in [-T,T]$.
\end{Th}

\begin{proof}
	Note that $f$ is complete since $M$ is compact.
	Let $g$ be the unique lift of $f$ via $D$, i.e., $g(y)=(d_yD)^{-1}(f(D(y))$.
	From the latter formula it is clear that $g$ is $C^k$ if $f$ is and (also using the chain rule and Picard-Lindel\"{o}f theorem) that $D$ maps trajectories of $g$ to those of $f$, i.e., 
	\begin{equation}
		D(\Phi_g^t(y)) = \Phi_f^t (D(y))
	\end{equation}
	for all $y\in \R^\ell$ and $t\in \R$ such that the left side is defined.
	But the left side is defined for \emph{all}  $y$ and $t$ since $f$ is complete and hence the path lifting property of covering maps \cite[p.~60]{hatcher2002algebraic} implies that $g$ is  also complete.
	
	Next, let $j \coloneqq i+1$.
	Since $U_{j}$ is contractible and $D$ is a smooth covering map, there exists a smooth map $E_{j}\colon U_j\to \R^\ell$ that is a local section of $D$, i.e., $D\circ E_j = \id_{U_j}$.
	It follows that \eqref{eq:vf-autoencoding-large-time-1} is satisfied for all $x\in U_i$, $t\in \R$ with $E\colon \R^n\to \R^\ell$ defined to be any smooth extension of $E_j|_{\cl{U_j}}\colon \cl{U_j}\to \R^\ell$.
	Finally, \eqref{eq:vf-autoencoding-large-time-2} follows from \eqref{eq:vf-autoencoding-large-time-1} and the facts that $\Phi_f$ and $(t,y,g)\mapsto \Phi_g^t(y)$ are continuous and composition of continuous functions is continuous with respect to the compact-open topologies \cite[p.~64, Ex.~10(a)]{hirsch1994differential}.
\end{proof}

While Corollary~\ref{co:dt-autoencoding-small-time} is not an entirely satisfactory analogue of Corollary~\ref{co:vf-autoencoding-small-time},  Theorem~\ref{th:vf-autoencoding-large-time-cover} \emph{does} have the following satisfactory  generalization to the discrete-time case.

\begin{Th}\label{th:dt-autoencoding-large-time-cover}
	Assume that $m = \ell$ and $M$ is a closed manifold for which there exists a smooth covering map $D\colon \R^\ell \to M$.
	Let $U_1,U_2,\ldots$ be as in Theorem~\ref{th:positive}.
	Let $F\colon M\to M$ be a $C^k$ map with $k\in \N_{\geq 0}\cup \{\infty\}$.
	Then there is a $C^k$ map $G\colon \R^\ell \to \R^\ell$ such that, for any $i\in \N$, there is a smooth map $E\colon \R^n\to \R^\ell$ (independent of $F$ and $G$) satisfying  $D\circ E|_{U_i} = \id_{U_i}$ and
	\begin{equation}\label{eq:dt-autoencoding-large-time-1}
		D\circ G^{\circ n}\circ E(x) = F^{\circ n}(x)
	\end{equation}
	for all $x\in U_i$ and $n\in \N$.
	Moreover, $G$ is a homeomorphism (resp. diffeomorphism) if $F$ is, and $G$ is a local homeomorphism (resp. local diffeomorphism) if $F$ is.
	
	In particular, for any subsets
\optionaldisplay{\E\subset C(\R^n,\R^\ell), \;\D\subset C(\R^\ell,\R^n), \;\mathcal{G} \subset C(\R^\ell,\R^\ell)}
that are dense in the compact-open topologies and any given $\varepsilon > 0$ and $N\in \N$, there are $\tilde{E}\in \E$, $\tilde{D}\in \D$, $\tilde{G}\in \mathcal{G}$ such that
	\begin{equation}\label{eq:dt-autoencoding-large-time-2}
		\|\tilde{D}\circ \tilde{G}^{\circ n} \circ \tilde{E}(x) - F^{\circ n}(x)\|< \varepsilon
	\end{equation} 
	for all $x\in U_i$ and $n\in \{0,\ldots, N\}$.
\end{Th}

\begin{proof}
	Fix a basepoint $x_0\in M$ and fix a point $\tilde{x}_0\in D^{-1}(x_0)\subset \R^\ell$. 
	Since $\R^\ell$ is simply connected, the lifting criterion for the covering space $D\colon \R^\ell \to M$ \cite[Prop.~1.33]{hatcher2002algebraic} applied to the map $F\circ D\colon \R^\ell \to M$ furnishes a unique \cite[Prop.~1.34]{hatcher2002algebraic} continuous map $G\colon \R^\ell \to \R^\ell$ making the diagram
	\begin{equation}\label{eq:cd-F-G}
		\begin{tikzcd}
			(\R^\ell, \tilde{x}_0) \arrow[dashed, "G"]{r} \arrow["D"]{d} & (\R^\ell, \tilde{x}_0) \arrow["D"]{d}\\
			(M,x_0) \arrow["F"]{r} & (M,x_0)
		\end{tikzcd}
	\end{equation}
	commute in the sense of pointed spaces and maps.
	
	To see that $G$ is $C^k$ if $F$ is, it suffices to fix any $y\in \R^\ell$ and show that $G$ is $C^k$ on some neighborhood of $y$.
	Set $x = D(G(y))$.
	Since $D$ is a smooth covering map, $x$ has a relatively open neighborhood $U\subset M$ such that $D$ restricts to a diffeomorphism from each connected component of $D^{-1}(U)$ onto $U$.
	Fix such a connected component $V$ and let $\sigma\colon U\to V$ be the unique diffeomorphism inverting $D|_V\colon V\to U$.
	Then $V'\coloneqq G^{-1}(V)$ is an open neighborhood of $y$ and
\begin{equation*}
		G|_{V'} = \sigma \circ F \circ D|_{V'}, 
\end{equation*}
	so that $G|_{V'}$ and hence also $G$ are $C^k$ if $F$ is.
	The same argument shows that $G$ is a local homeomorphism (resp. local diffeomorphism) if $F$ is.
	
	To see that $G$ is a (global) homeomorphism if $F$ is, the same argument from the first paragraph of the proof furnishes a  continuous (and $C^k$ if $F$ is) map $H\colon \R^\ell \to \R^\ell$ making the diagram
	\begin{equation}\label{eq:cd-Finv-H}
	\begin{tikzcd}
		(\R^\ell, \tilde{x}_0) \arrow[dashed, "H"]{r} \arrow["D"]{d} & (\R^\ell, \tilde{x}_0) \arrow["D"]{d}\\
		(M,x_0) \arrow["F^{-1}"]{r} & (M,x_0)
	\end{tikzcd}
\end{equation}
commute in the sense of pointed spaces and maps.
	Horizontally concatenating the diagrams \eqref{eq:cd-F-G}, \eqref{eq:cd-Finv-H} in both orders yields the pair of pointed commutative diagrams 
		\begin{equation*}
		\begin{tikzcd}
			(\R^\ell, \tilde{x}_0)\arrow[dashed, "G\circ H"]{r} \arrow["D"]{d} & (\R^\ell, \tilde{x}_0) \arrow["D"]{d}\\
			(M,x_0) \arrow["\id_M"]{r} & (M,x_0)
		\end{tikzcd}
		\quad \textnormal{and} \quad
		\begin{tikzcd}
			(\R^\ell, \tilde{x}_0) \arrow[dashed, "H\circ G"]{r} \arrow["D"]{d} & (\R^\ell, \tilde{x}_0) \arrow["D"]{d}\\
			(M,x_0) \arrow["\id_M"]{r} & (M,x_0)
		\end{tikzcd}.
	\end{equation*}
	Notice that these latter diagrams would also commute in the pointed sense if $G\circ H$ and $H\circ G$ were both replaced by $\id_{\R^\ell}$.
	Thus, the uniqueness portion \cite[Prop.~1.34]{hatcher2002algebraic} of the lifting criterion for covering spaces implies that
\begin{equation}\label{eq:H-is-Ginv}		
	G\circ H = \id_{\R^\ell} = H\circ G,
\end{equation}
	 so that $G$ is a homeomorphism with $G^{-1}=H$ if $F$ is a homeomorphism.
	 And since $G$, $H$ are $C^k$ if $F$ is $C^k$,  
     equality
     \eqref{eq:H-is-Ginv} 
     also implies that $G$ is a $C^k$ diffeomorphism if $F$ is.
	 
 	Finally, \eqref{eq:dt-autoencoding-large-time-2} follows from \eqref{eq:dt-autoencoding-large-time-1} and the fact that composition of continuous functions is continuous with respect to the compact-open topologies \cite[p.~64, Ex.~10(a)]{hirsch1994differential}.
\end{proof}

\begin{Rem}\label{rem:GAS}
Suppose that $x_0\in M$ is an asymptotically stable equilibrium of a complete vector field $f$, and let $\mathcal{D}$ be the domain of attraction of $x_0$. As a submanifold, $\mathcal{D}$ is diffeomorphic to $\R^m$. Moreover, there is a \textit{global} (on $\mathcal{D}$) topological conjugacy to a \textit{linear} system in $\R^m$. This construction can be interpreted as a perfect autoencoder---restricted to $\mathcal{D}$---such that the dynamics in the latent space are linear.  Furthermore, this linearizing encoder/decoder pair is in fact a $C^{k\geq 1}$ diffeomorphism on $\mathcal{D}\setminus\{x_0\}$ provided that the vector field is $C^k$ and the underlying space is not $5$-dimensional. (The $C^k$ statement in the $5$-dimensional case is equivalent to the still-open $4$-dimensional smooth Poincar\'{e} conjecture.) See~\cite{kvalheim_sontag_grobman} for details.
\end{Rem}

\section{Proof of Theorem~\ref{th:positive}}\label{sec:proofs}

We now restate and prove Theorem~\ref{th:positive}, which generalizes \cite[Thm~1]{kvalheim2024why}. 
Recall that we are assuming for simplicity that $M$ is connected.

\ThmPos*

\begin{proof}
The proofs of \cite[Lem.~1, 2]{kvalheim2024why} produce a closed subset $C\subset M$ and smooth embedding $E_0\colon M\setminus C\to \R^\ell$ such that $C$ is disjoint from $S$, $C$ is measure zero in $M$, and $C\cap \partial M$ is measure zero in $\partial M$.\footnote{These proofs take $C$ to be the finite union of ascending discs of positive codimension for a suitable polar Morse function on $M$ with negative gradient $F$.}
Moreover, the same proofs yield a point $p \in M$ and a smooth vector field $F$ on $M$ pointing inward at $\partial M$ whose induced  semiflow $\Phi\colon [0,\infty)\times M\to M$  is such that $\Phi^{t}(M\setminus C)\subset M\setminus C$ for all $t\geq 0$ and $p$ is a hyperbolic asymptotically stable equilibrium for $\Phi$ with basin of attraction $M\setminus C$.

Let $\Psi\colon \R\times \R^n\to \R^n$ be the flow of any compactly supported smooth vector field $G$ on $\R^n$ extending $F$, so that $\Psi|_{[0,\infty)\times M}=\Phi$.
Define $\tilde{M}\coloneqq \Psi^{-1}( M)\supset M$ and $\tilde{C}\coloneqq \Psi^{-1}(C)\supset C$.
Fix a smooth function $\rho\colon \R^n\to [0,\infty)$ satisfying $\rho^{-1}(0)=\partial \tilde{M}$.
Define the vector field $\tilde{F}\coloneqq (\rho G)|_{\tilde{M}}$ which, by construction and compactness of $\tilde{M}$, induces a well-defined smooth flow $\tilde{\Phi}\colon \R\times \tilde{M}\to \tilde{M}$.
Moreover, $p$ is asymptotically stable for $\tilde{\Phi}$ with basin of attraction 
\begin{equation}\label{eq:basin}
B= \tilde{M}\setminus (\tilde{C}\cup \partial \tilde{M}).	
\end{equation}

Since $p$ is hyperbolic, there is a smoothly embedded closed $m$-disc $V\subset M$ containing $p$ in its relative interior $\interiortM{V}$ such that $F|_V$ and hence also $G|_V$ are inward pointing at $\partial V$.
For each $i\in \N$, set 
$V_i\coloneqq \tilde{\Phi}^{-i}(V)$ so that $\interiortM{V_i}= \tilde{\Phi}^{-i}(\interiortM{V})$, and define
\begin{equation}\label{eq:U_i-def}
	U_i = M\cap \interiortM{V_i}.
\end{equation}
Note that each $U_i$ is relatively open in $M$ and that 
\begin{equation}\label{eq:cl-U_i-expression}
	\cl{U_i}=M\cap V_i.
\end{equation} 
Since $G|_V$ is inward pointing at $\partial V$, it follows that $V_i \subset \interiortM{V_{i+1}}$ and hence also $\cl{U_i} \subset U_{i+1}$ for all $i\in \N$.
And if $\partial M = \varnothing$, then $\tilde{M}=M$ and \eqref{eq:cl-U_i-expression} reduces to $\cl{U_i}=V_i$, which is a smooth manifold diffeomorphic to the $m$-disc $V$.

Observe that $M \setminus C = \bigcup_{i\in \N}U_i$ since each point in this set enters $V$ when flowing via $\Phi$ after some positive time,
and the fact that $S\subset M \setminus C$ is a finite set implies that there is $i_0\in \N$ such that $S\subset U_i$ for all $i> i_0$.
Thus, after discarding finitely many $U_i$, we may assume that $S\subset U_i$ for all $i$.

We now show that  $U_i$ and $\cl{U_i}$ are contractible (if $\partial M = \varnothing$, this is immediate from the disc statement above).
Since $\tilde{\Phi}^t(M)\subset M$ for all $t\geq 0$, \eqref{eq:basin} implies that $H\colon [0,1]\times (M\setminus C)\to M\setminus C$ defined by 
\begin{equation}\label{eq:homotopy}
	H(t,x)\coloneqq 
	\begin{cases}
		p, & t=1\\
		\tilde{\Phi}^{\frac{t}{1-t}}(x), & 0\leq t < 1	
	\end{cases}	
\end{equation}
is a well-defined homotopy from $\id_{M\setminus C}$ to the constant map $M\setminus C \to \{p\}$ yielding a deformation retraction of $M\setminus C$ to $\{p\}$.
And since $G|_V$ is inward pointing at $\partial V$, it follows that  $G|_{V_i}$ is inward pointing at $\partial V_i$.
Thus, $\tilde{\Phi}^t(V_i)\subset \interiortM{V_i}$ and hence also $\tilde{\Phi}^t(\cl{U_i})\subset \cl{U_i}$, $\tilde{\Phi}^t(U_i)\subset U_i$ for all $i\in \N$ and $t\geq 0$.
Thus, \eqref{eq:homotopy} restricts to well-defined strong deformation retractions of $U_i$ and of $\cl{U_i}$ to $\{p\}$ for all $i$, demonstrating that each $U_i$ and $\cl{U_i}$ are contractible.

Now fix $i\in \N$.
Let $E\colon \R^n\to \R^\ell$ be any smooth extension to $\R^n\to \R^\ell$ of $E_0|_{\cl{U_i}}$, and let $D\colon \R^\ell\to \R^n$ be any smooth extension of $E|_{\cl{U_i}}^{-1}\colon E(\cl{U_i})\to \cl{U_i}\subset \R^n$.
Then
\begin{align*}
\DE|_{U_i}
&
= 	E|_{\cl{U_i}}^{-1} \circ E|_{U_i} 
\\
&
= \id_{U_i},
\end{align*}
thus completing the proof,
\end{proof}

\section{A simple example}

We consider the dynamical system in $\R^2$ given by the following
vector field:
\[
f(x)=
\begin{bmatrix}
-2x_1x_2^2 \\[4pt]
\ 2x_1^2x_2
\end{bmatrix}
\]
Observe that the coordinate axes $x_1=0$ and $x_2=0$ constitute equilibria.
The data manifold $M$ is the unit circle \(S^1\), which is invariant
for the vector field since the normal at any point $(x_1,x_2)$ on the
circle is given by $(x_1,x_2)$, and $(x_1,x_2)^\top f(x) = 0$.
Using the polar angle $\theta$ as a coordinate, so
\(x_1=\cos\theta\) and \(x_2=\sin\theta\), the vector field
restricted to \(S^1\) is described by the differential equation
\[
\dot{\theta} = \sin(2\theta),\quad \theta\in[0,2\pi)\,.
\]
Observe that on \(S^1\) the vector field has equilibria at
\(\theta\in\{0,\tfrac{\pi}{2},\pi,\tfrac{3\pi}{2}\}\).
Linearization gives 
\[
\left.\frac{d}{d\theta}\dot\theta\right|_{\theta=\theta^\star} \;=\; 2\cos(2\theta^\star),
\]
so \(\theta=\tfrac{\pi}{2},\tfrac{3\pi}{2}\) (``north'' and ``south''
points respectively) are asymptotically stable,
and \(\theta=0,\pi\) (``east'' and ``west''
points respectively) are unstable.
Figure~\ref{fig:fig_flow_S1} illustrates the flow restricted to \(S^1\).

\begin{figure}[ht]
\picc{1}{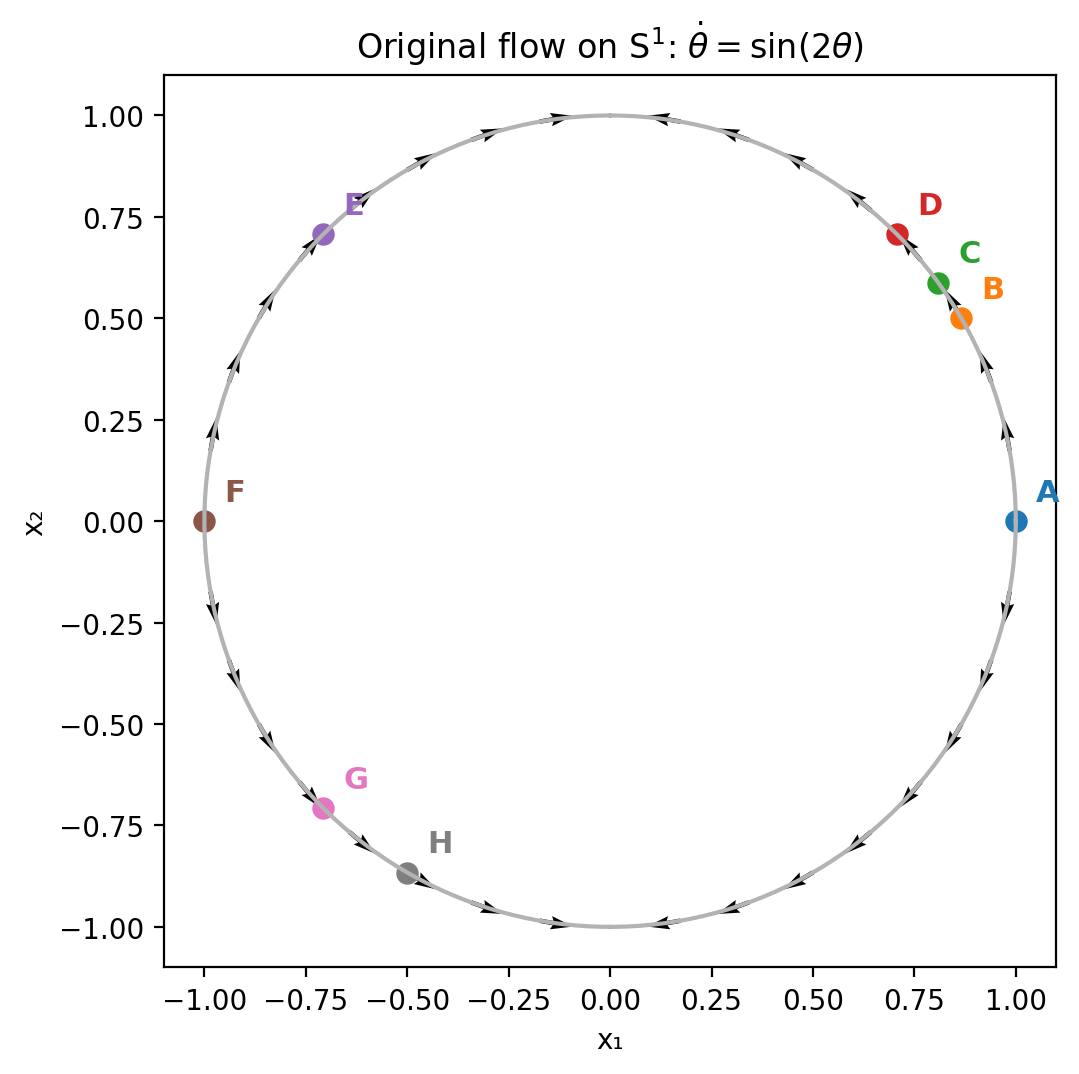}
\caption{Original flow when restricted to data manifold. We sampled \(\theta\) uniformly, drew the circle
  \(x(\theta)\), and overlaid normalized arrows \(f(x(\theta))\)
  tangential to the circle. A set of particular labeled points $A,\ldots,H$ are marked at their \((\cos\theta,\sin\theta)\) locations for
  future reference.} 
\label{fig:fig_flow_S1}
\end{figure}

\subsection{Autoencoding through $\R$}

Our objective is to autoencode the dynamics on the data manifold $S^1$
through the one-dimensional Euclidean latent space $L=\R$. For this
purpose, we train a deep neural network or ``multilayer perceptron''
(MLP) architecture with no prior information about the structure of
the data manifold, but using samples from trajectories.  (We will use
``$\phi$'' as the name of the state variable in $\R$ so as to remind
ourselves that this variable might for example represent a chart on
the data manifold.)

Thus, we want to learn a one-dimensional latent ODE
\( \dot{\phi}=h(\phi) \)
together with a neural encoder/decoder pair of maps
\(x\mapsto \phi=E(x), \ \hat{x}=D(\phi)\)
so that the latent vector field approximates the push-forward under
$E$ of the vector field on the unit circle and, conversely, the
push-forward under $D$ of the latent vector field approximates the
vector field on the unit circle.

The encoder \(E:\mathbb{R}^2\to\mathbb{R}\),
decoder \(D:\mathbb{R}\to\mathbb{R}^2\),
and latent ODE will each be given by an MLP.
While our theory results assure conjugacy missmatches that are
uniformly small on large domains, neural network training methods and software
work best for mean square (or similar) losses on entire domains,
in a way similar to what we described in Corollary~\ref{co:lp-intro}
for encoding the manifold itself. Therefore, we will employ the latter
in our algorithms.
Specifically, the parameters (weights) of these networks were obtained by minimizing
a loss function which is made up of three terms:
(1) a penalty on the data manifold reconstruction error
(encoding/decoding round trip should be close to the identify), 
(2) a penalty on the conjugacy mismatch between the push-forward of
the latent vector field and the vector field on $S^1$, and 
(3) a penalty on the mismatch between samples of the original dynamics
mapped under $E$ and samples obtained from the latent dynamics.
More details are given below.
(Of course, other penalties could be used, but this method was found
to work well.)

\subsection{Neural architectures and 2-step training}

\textbf{Encoder \(E:\mathbb{R}^2\to\mathbb{R}\):} This is an MLP with
\(\tanh\) activation and the following layers:
\[
2 \to 128 \to 128 \to 128 \to 1 \quad(\text{bias in all layers}).
\]
\textbf{Decoder \(D:\mathbb{R}\to\mathbb{R}^2\):} This is an MLP with
\(\tanh\) activation and the following layers:
\[
1 \to 128 \to 128 \to 2.
\]
\textbf{Latent ODE \(h:\mathbb{R}\to\mathbb{R}\):} This is an MLP with
\(\tanh\) activation and the following layers:
\[
1 \to 64 \to 64 \to 1.
\]

Instead of attempting to simultaneously both autoencode the manifold
and learn the dynamics, we found it more efficient to use a
\textit{two-phase} procedure: (i) autoencoder pretraining with only
reconstruction loss, and then (ii) gradual annealing of conjugacy and
latent one-step losses, while still maintaining the autoencoder property.

We generate data from the true dynamical system on \(S^1\)
$\dot{\theta} = \sin(2\theta)$, $\theta\in[0,2\pi)$
by discretizing with step \( \Delta t=0.04 \) and using one-step
forward Euler:
\[
\theta_{t+1}=\theta_t+\Delta t\,\sin(2\theta_t).
\]
(One could employ better solvers, but this approximation already works well.)
We sample \(512\) trajectories of length \(96\) with
initial conditions \(\theta_{i,0}\sim\mathrm{Unif}[0,2\pi)\), and store
\[
X[i,t]=(\cos\theta_{i,t},\sin\theta_{i,t})\in\mathbb{R}^2, \;
i\in\{1,\ldots, 512\}, \; t\in\{0,\ldots, 96\}
\]
and the following iterate pairs for each of the 512 trajectories:
\[
X_t=X[:,0{:}95],\quad X_{t+1}=X[:,1{:}96].
\]
In addition, eight particular initial conditions \(A\ldots H\) are
labeled and used for color-coded visualizations, as follows:
\[
\text{A}:0,\ \text{B}:\tfrac{\pi}{6},\ \text{C}:\tfrac{\pi}{5},\ \text{D}:\tfrac{\pi}{4},\
\text{E}:\tfrac{3\pi}{4},\ \text{F}:\pi,\ \text{G}:\tfrac{5\pi}{4},\ \text{H}:\tfrac{4\pi}{3}.
\]

\subsection{Results}

Figure~\ref{fig_phi_of_theta} shows the learned encoder, and the image
of the unit circle mapped into an interval in $\R$.
It is noteworthy that the map is one to one on a large part of the unit
circle, roughly on $[0.014,2\pi-0.014]$.
\begin{figure}[ht]
\picc{1}{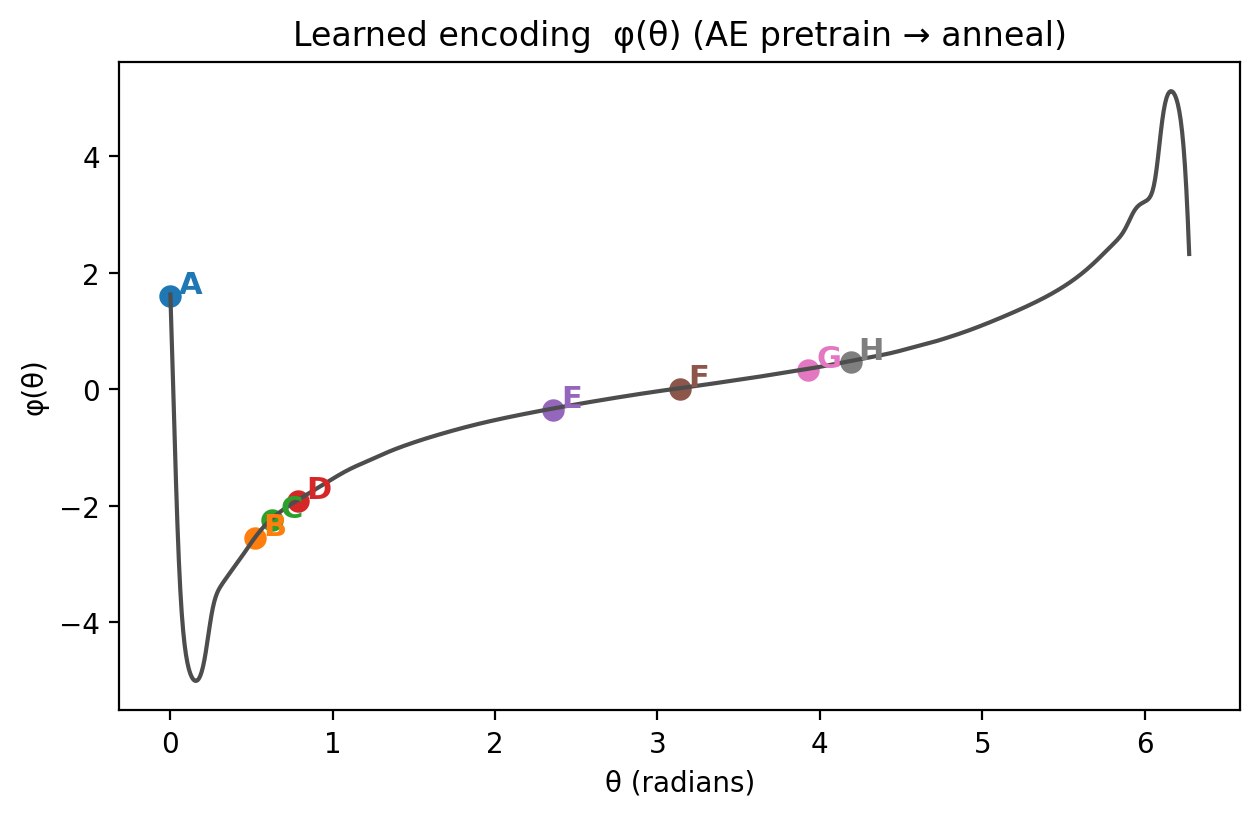}
  \caption{Learned encoder \(\phi(\theta)\) with labels A--H.}
\label{fig_phi_of_theta}
\end{figure}

Figure~\ref{fig_latent_vf} shows the computed latent vector
field. Note that at the points $B,C,D,G,H$ and at the point $E$ the
vector field is positive and negative respectively 
which is consistent with the counterclockwise and clockwise motions in
the data manifold, respectively.
The equilibrium at $F=\pi$ is mapped into an approximate equilibrium in the latent dynamics.
At the image of $A=0$ there should have been an equilibrium, but
instead the vector field points in the negative direction, which is
consistent with a cut near $A$ when building the encoder (compare
the start and end of the encoding map).
\begin{figure}[ht]
\picc{1}{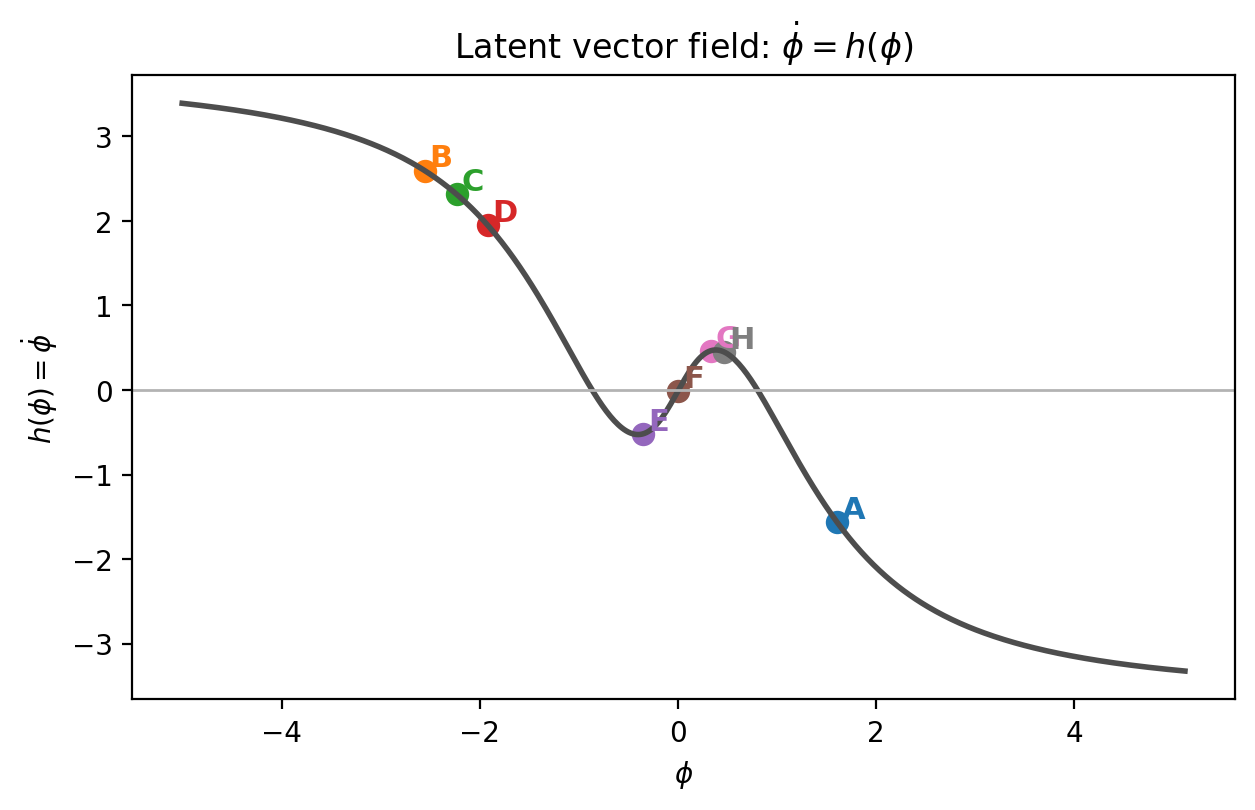}
\caption{Latent vector field \(h(\phi)\) with A--H overlays.}
\label{fig_latent_vf}
\end{figure}

Figure~\ref{fig_decoder_R2.png} shows an overlapping picture of the
data manifold $S^1$ (light gray) together with its image (black) under
the ``round trip'' map $x\mapsto D(E(x))$. Shown as well are the
images $A'$, $B'$, $\ldots$ of the various sample points
$A,B,\ldots$. Note the cut around $A=(0,1)$ (or, in polar angles, $A=0$).
\begin{figure}[ht]
\picc{1}{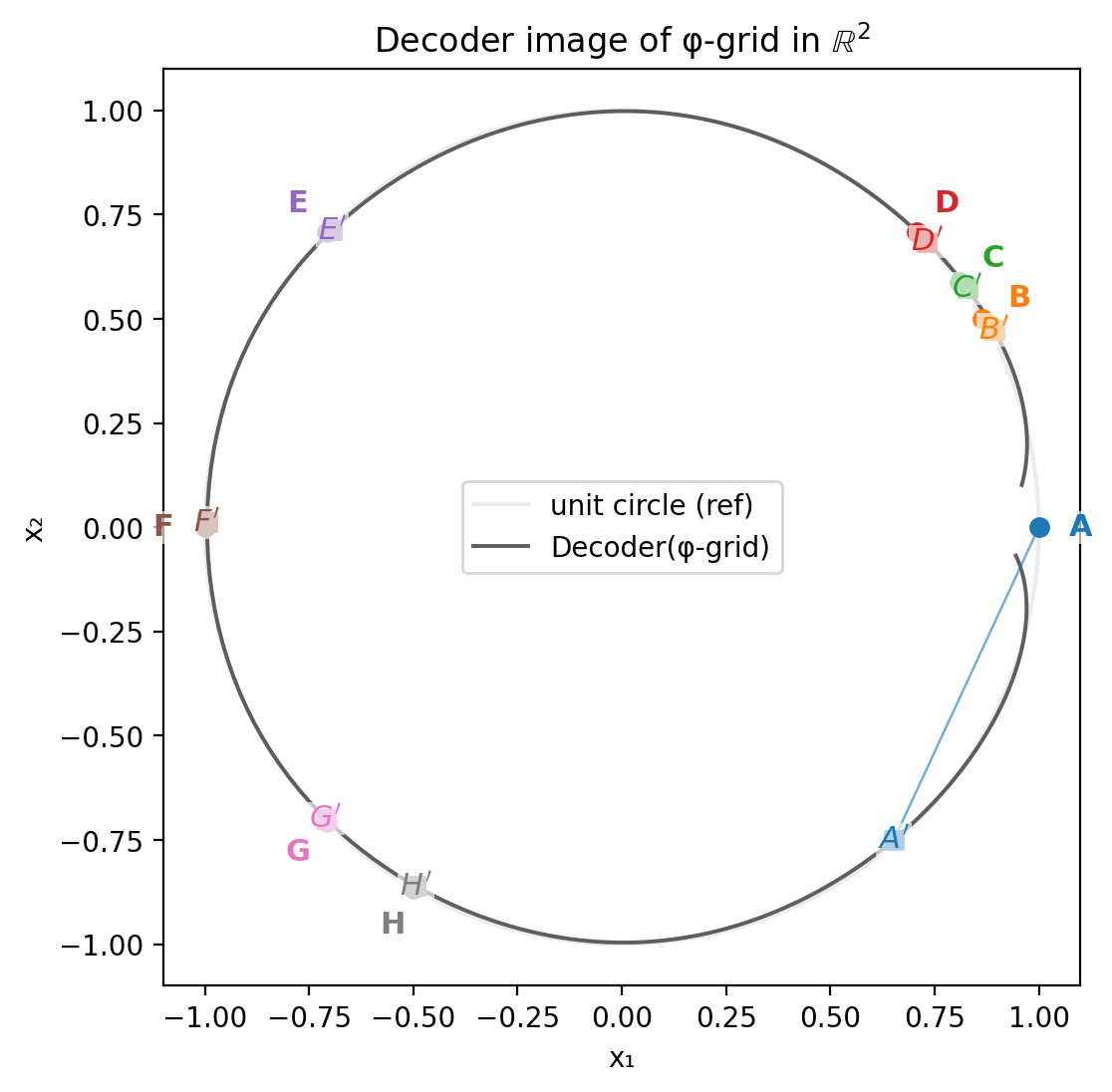}
\caption{Decoder image \(D(\phi)\) in \(\mathbb{R}^2\). Small circles
 how the original points on the unit circle labeled $A\ldots H$, and
 small squares show the respective decoded points \(A'\ldots H'\).}
\label{fig_decoder_R2.png}
\end{figure}

Visually, it would appear in Figure~\ref{fig_decoder_R2.png}
that the images of $B,\ldots H$ lie exactly
on the unit circle, but this is not the case. The autoencoder only forces close
proximity to the circle.
In fact, The magnitudes of the points $A',\ldots H'$ are,
respectively, approximately:
\[
1.000953,
1.001267,
0.999459,
0.996480,
0.999774,
0.992573,
0.999803,
0.997436 .
\]
Figure~\ref{fig_decoder_radius} plots these magnitudes.
\begin{figure}[ht]
\picc{0.85}{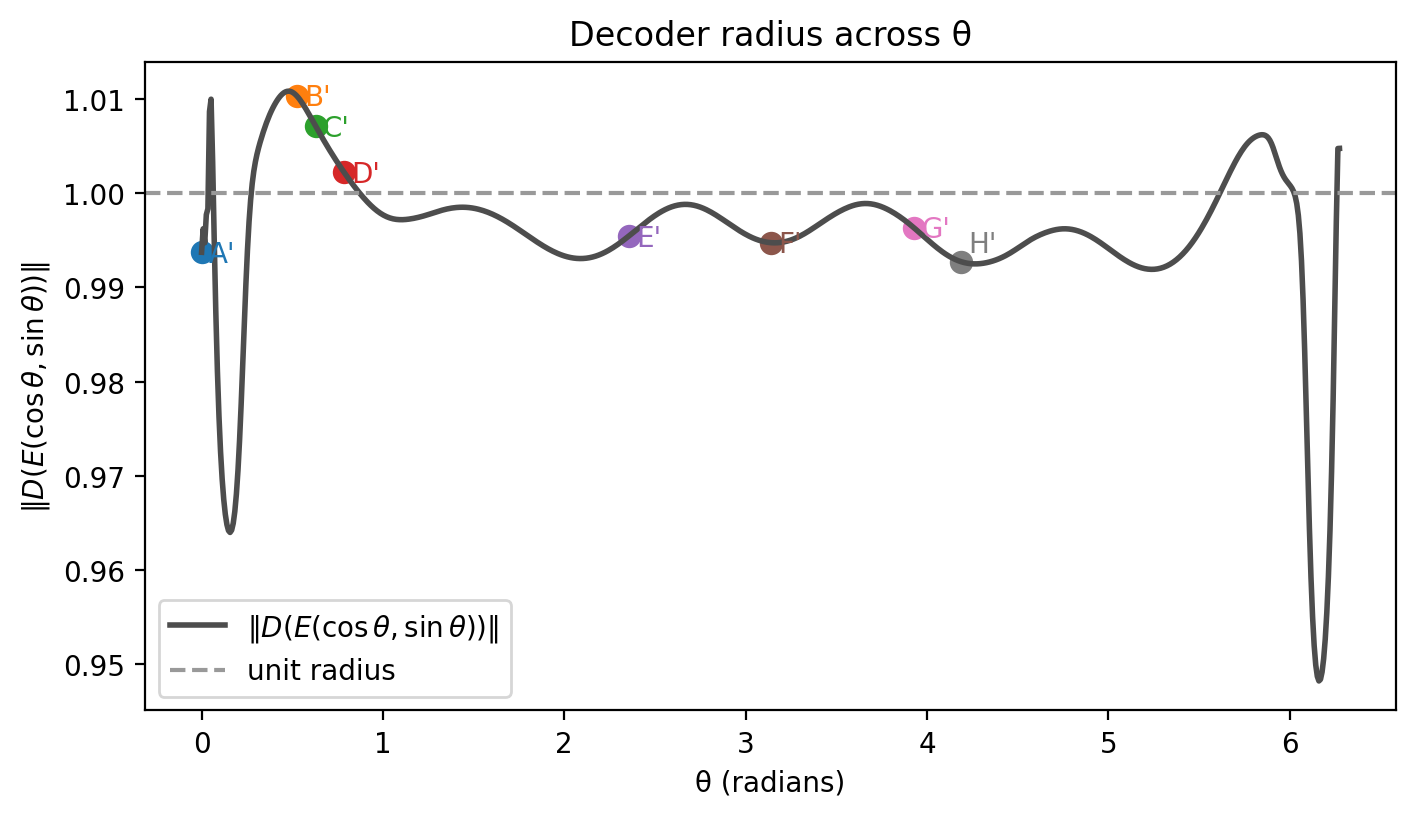}
\caption{Decoder magnitude mapping.}
\label{fig_decoder_radius}
\end{figure}
  
Figure~\ref{fig_decoder_angle} shows the polar angles (projection onto
the unit circle) of the image of the decoder, to help understand the
decoder behavior.
\begin{figure}[ht]
\picc{0.9}{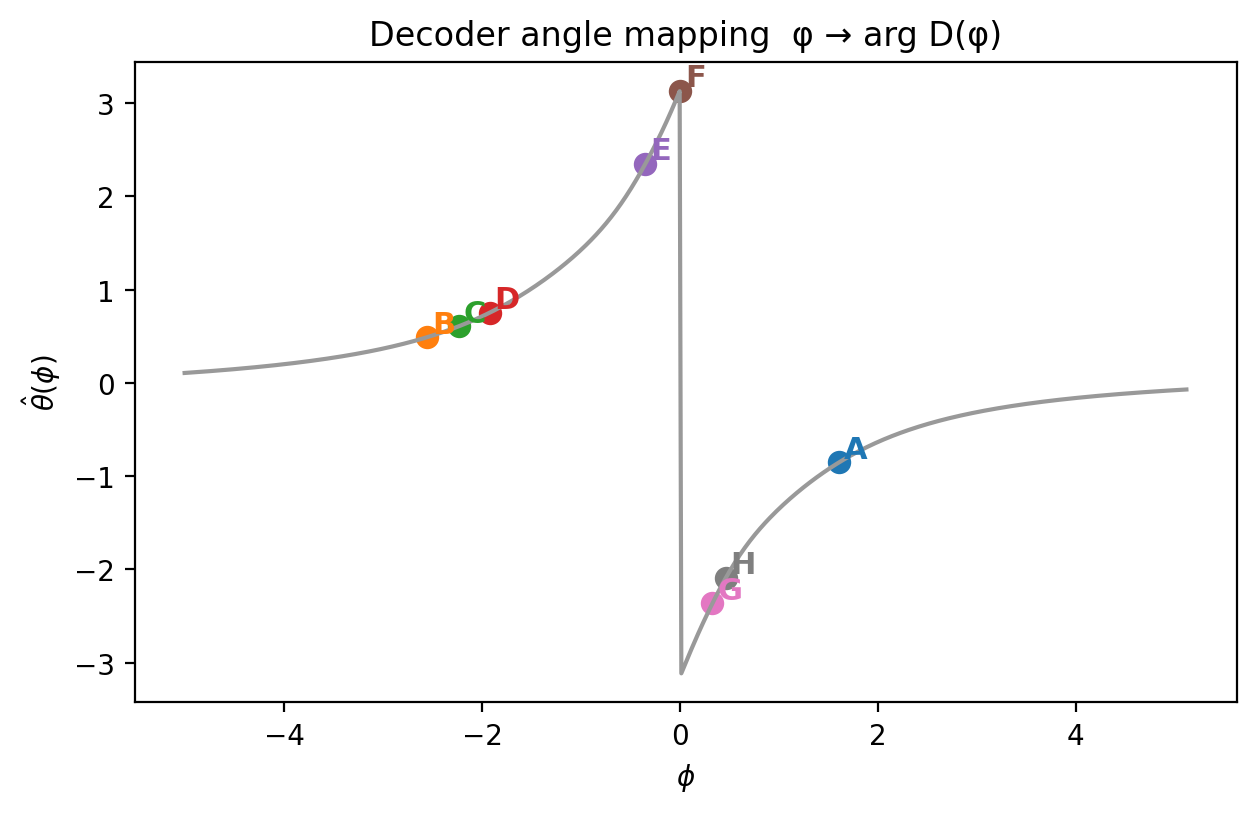}
\caption{Decoder angle mapping \(\phi\mapsto\hat{\theta}(\phi)=\arg D(\phi)\).}
\label{fig_decoder_angle}
\end{figure}
  
Figure~\ref{fig_thetaVF_pulled_zoom} shows on the same window both the
original vector field and the pull-back of the latent vector field.
The reconstruction is almost perfect except near $A$.
\begin{figure}[ht]
  \picc{0.8}{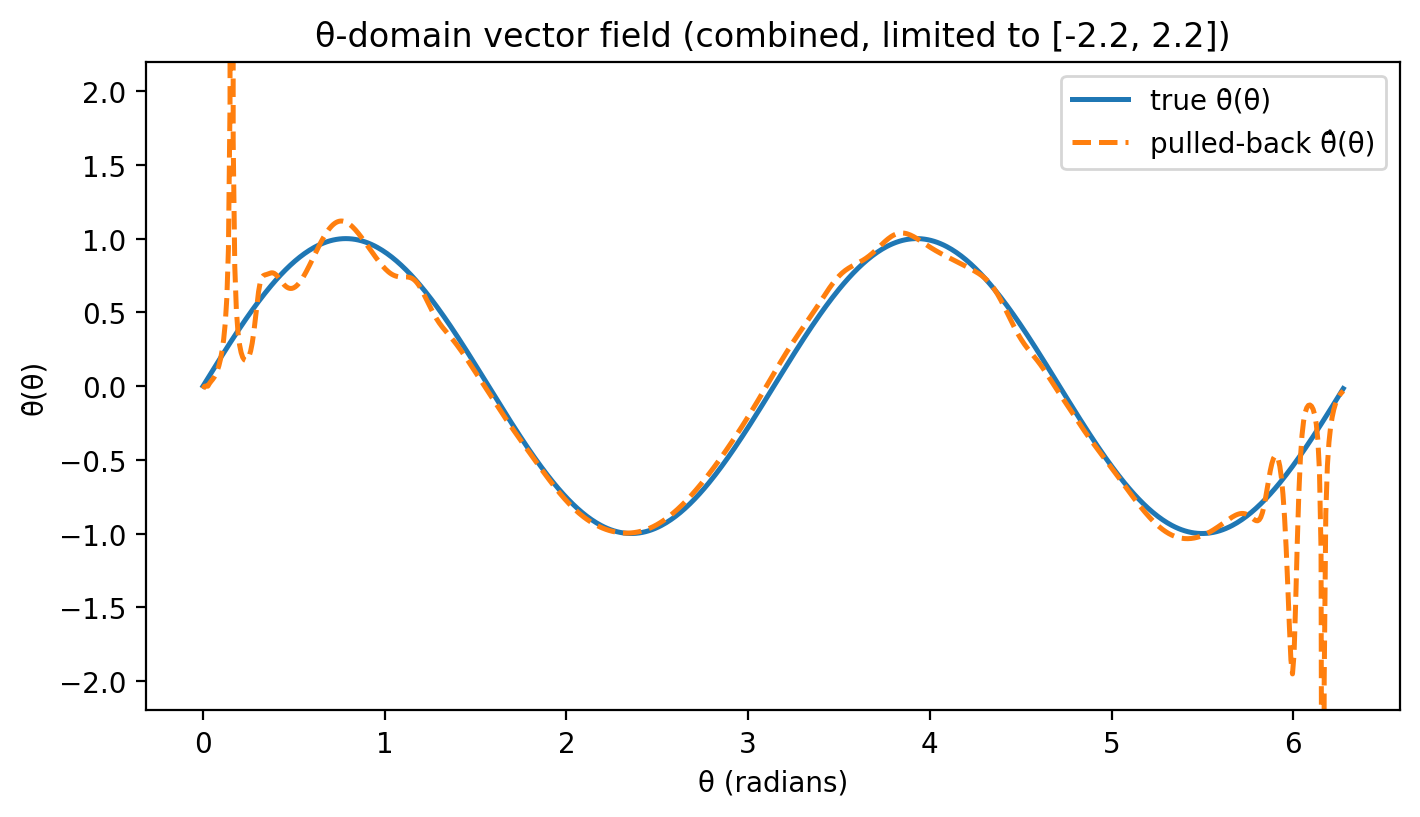}
  \caption{Original vector field and pull-back of latent vector field.}
\label{fig_thetaVF_pulled_zoom}
\end{figure}

For legibility, overflow and underflow values in
Figure~\ref{fig_thetaVF_pulled_zoom} were truncated to the window
$[-1.2,1.2]$. The full pullback is shown in Figure~\ref{fig_thetaVF_pulled}. 
\begin{figure}[ht]
\picc{0.8}{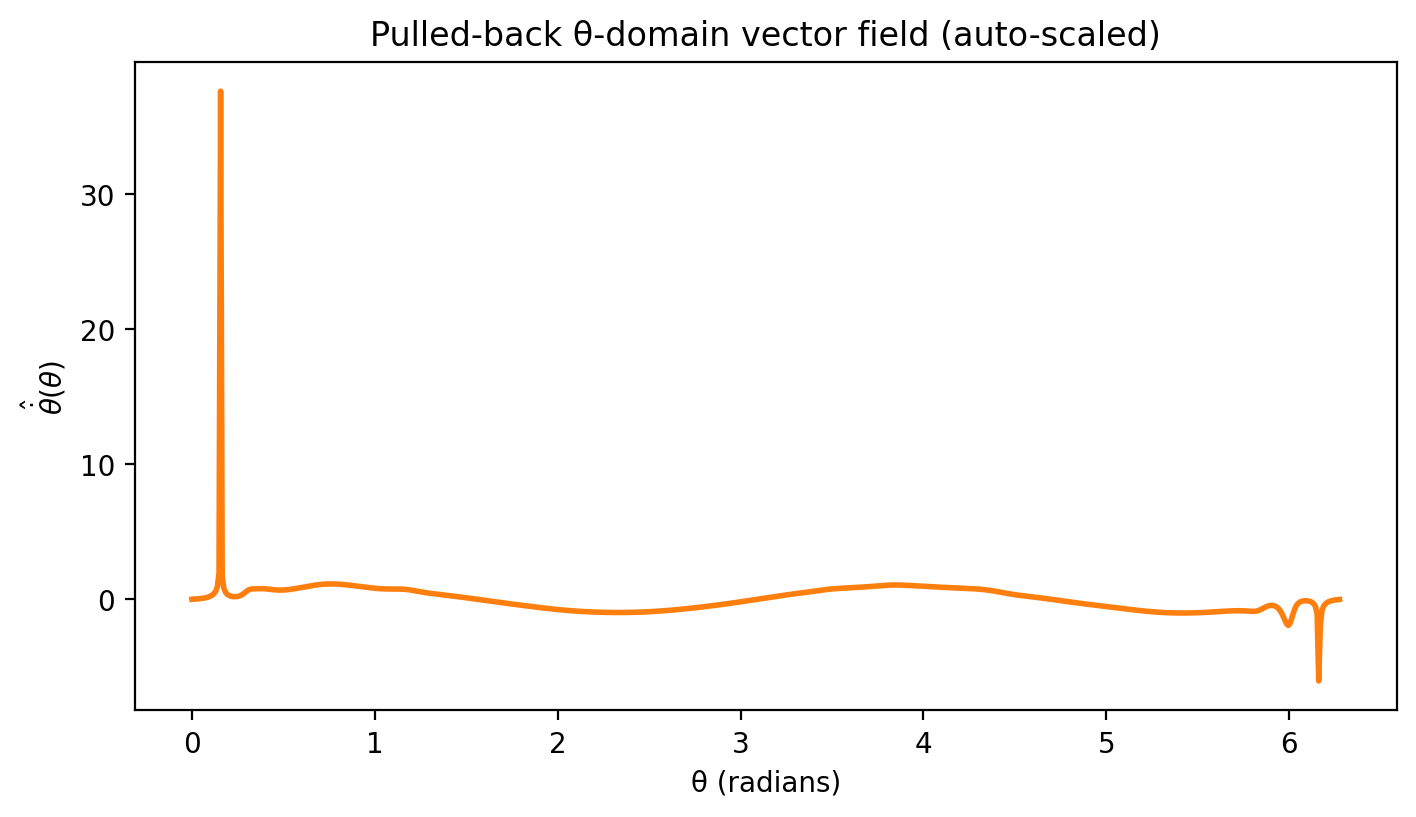}
\caption{
  Pull-back of latent vector field.}
\label{fig_thetaVF_pulled}
\end{figure}

Figure~\ref{fig_timeseries_AH} shows time series (true and encoded/decoded, as
well as image of trajectories in latent space).
\begin{figure}[ht]
\picc{.45}{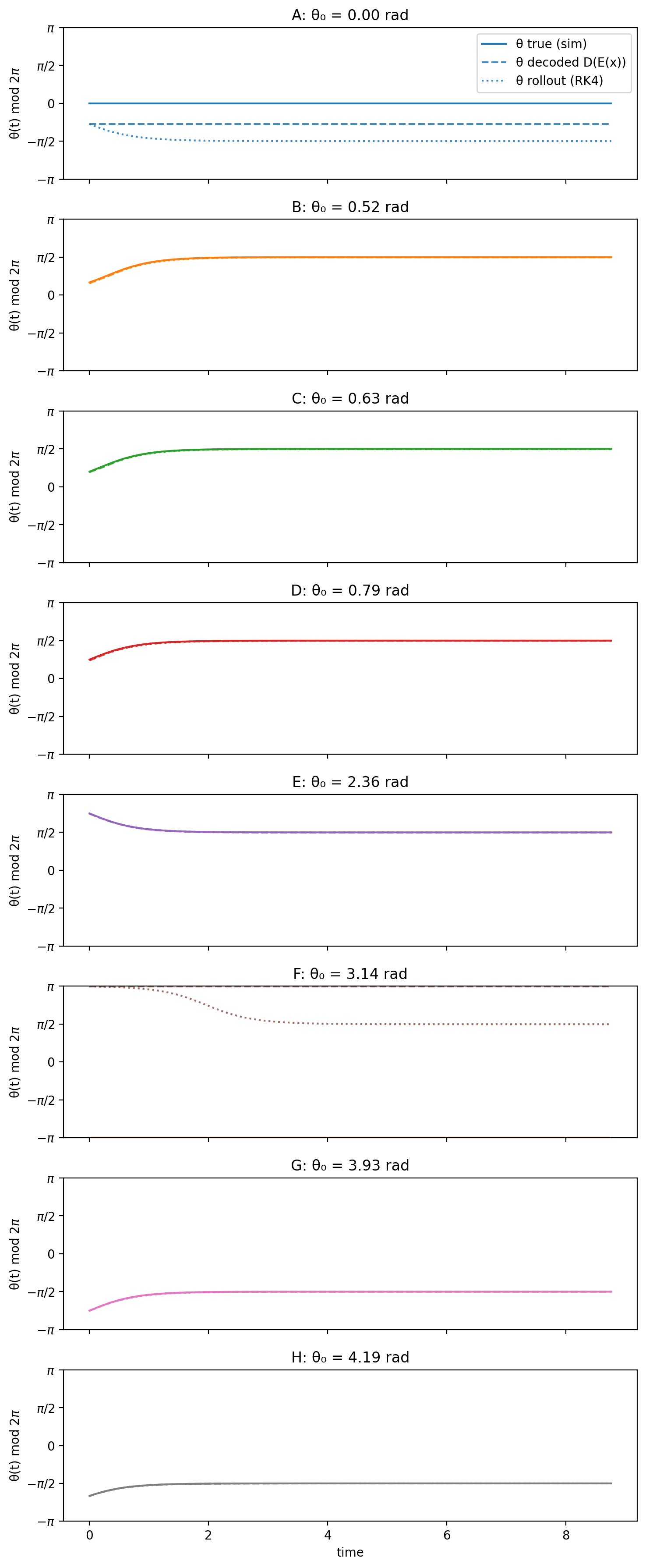}
\caption{Time-series panels for A\ldots H (true / decoded / rollout).
For each tag \(\theta_0\):
\textit{true}: simulate \(\theta_t\) by Euler on \(\dot\theta=\sin(2\theta)\);
\textit{decoded}: at each true \(x_t=(\cos\theta_t,\sin\theta_t)\) compute \(\hat x_t=D(E(x_t))\) and take \(\hat\theta_t=\arg \hat x_t\);
\textit{rollout}: start from \(x_0\), encode \(\phi_0=E(x_0)\), integrate latent ODE by RK4 to \(\phi_t\), decode \(x^{\text{roll}}_t=D(\phi_t)\), and take \(\theta^{\text{roll}}_t=\arg x^{\text{roll}}_t\).
We plot each series \textit{wrapped modulo \(2\pi\)} into
\((-\pi,\pi]\).}
\label{fig_timeseries_AH}
\end{figure}

\clearpage
\subsection{Methods}

\subsubsection{Loss function}

The loss function has the following components.

\textit{(1) Reconstruction in \(\mathbb{R}^2\).}
\[
L_{\mathrm{rec}}=\|D(E(x))-x\|_2^2.
\]

\textit{(2) Conjugacy (push-forward match).}

Let \(\phi=E(x)\), \(\hat{x}=D(\phi)\).
Thus the recovered vector field is
\[
f(\hat{x})=
\begin{bmatrix}
-2\hat{x}_1\hat{x}_2^2 \\[4pt]
\ 2\hat{x}_1^2\hat{x}_2
\end{bmatrix}
\in\mathbb{R}^{2\times 1}.
\]
Finally, we compute \(J_D(\phi)\) by
PyTorch's automatic differentiation engine autograd, and define
\[
v_{\mathrm{push}}=J_D(\phi)\,h(\phi)\in\mathbb{R}^2.
\]
Then
\[
L_{\mathrm{conj}}=\|v_{\mathrm{push}}-f(\hat{x})\|_2^2.
\]

\textit{(3) Latent one-step (RK4).}
For pairs \((x_t,x_{t+1})\),
\[
\phi_t=E(x_t),\quad \phi_{t+1}^{\rm enc}=E(x_{t+1}),\quad
\phi_{t+1}^{\rm pred}=\mathrm{RK4}(\phi_t,h;\Delta t).
\]
Here the numerical integration (RK4) is given by a standard algorithm:
Given \(\phi_t\) and step \(\Delta t\),
\[
\begin{aligned}
k_1&=h(\phi_t),\quad
k_2=h(\phi_t+\tfrac12\Delta t\,k_1),\quad
k_3=h(\phi_t+\tfrac12\Delta t\,k_2),\\
k_4&=h(\phi_t+\Delta t\,k_3),\qquad
\phi_{t+1}=\phi_t+\tfrac{\Delta t}{6}(k_1+2k_2+2k_3+k_4).
\end{aligned}
\]
The third loss term is:
\[
L_{\mathrm{lat1}}=\|\phi_{t+1}^{\rm pred}-\phi_{t+1}^{\rm enc}\|_2^2.
\]

\subsubsection{Training}

The two-phase training (annealed) procedure was as follows.
We optimize
\[
\min\ \ W_{\mathrm{rec}}L_{\mathrm{rec}}+W_{\mathrm{conj}}L_{\mathrm{conj}}+W_{\mathrm{lat1}}L_{\mathrm{lat1}}.
\]
\textit{Phase 1 (AE pretrain)}: \(500\) epochs, \(W_{\mathrm{rec}}=15.0\), \(W_{\mathrm{conj}}=W_{\mathrm{lat1}}=0\), AdamW with \(lr=2\!\times\!10^{-3}\), weight decay \(10^{-5}\).
\textit{Phases 2--4 (anneal dynamics)}:
\[
\begin{array}{c|c|c|c|c}
\text{Phase} & \text{Epochs} & W_{\mathrm{rec}} & W_{\mathrm{conj}} & W_{\mathrm{lat1}} \\ \hline
2 & 250 & 10.0 & 0.5 & 0.2 \\
3 & 250 & 7.0  & 1.0 & 0.5 \\
4 & 250 & 5.0  & 2.0 & 0.8
\end{array}
\quad \text{(AdamW, }lr\in\{1.5\!\times\!10^{-3}, 10^{-3}, 10^{-3}\}\text{).}
\]
Batch size \(4096\). Default dtype float32.

\section{Conclusions and Discussion}

This work has examined the theoretical foundations of autoencoders as representations of dynamical systems evolving on low-dimensional manifolds. Motivated by the empirical success of data-driven methods that discover intrinsic state variables from high-dimensional observations, we have provided a rigorous framework showing that, under suitable conditions, the dynamics on a smooth manifold \(M\) of dimension \(k\) can be faithfully interlaced with a latent dynamical system on \(\mathbb{R}^k\) through an encoder–decoder pair. This result formalizes the intuition that autoencoders can learn intrinsic coordinates consistent with the underlying system’s evolution, while also clarifying the topological obstructions that preclude a global construction. In particular, we have shown that such interlacing mappings can exist only on large subsets of \(M\), excluding regions of small measure where smooth global parametrizations are impossible.

Beyond the dynamical setting, this paper also contributes to the broader mathematical understanding of autoencoders as manifold learning devices. By analyzing the geometric and topological limitations inherent in the search for encoder–decoder pairs, we have delineated the boundary between what can be achieved in principle and what must necessarily fail, even in the absence of noise or optimization constraints. These insights extend the conceptual foundation established in our earlier work~\cite{kvalheim2024why}, offering a unified theoretical perspective that encompasses both static and dynamical data.

The results presented here open several directions for further research. On the theoretical side, one may explore finer characterizations of the regions where interlacing maps fail to exist, or connect the constructions developed here to the theory of embeddings and immersions from differential topology. On the applied side, our findings motivate the design of architectures and training objectives that explicitly respect manifold structure and local coordinate consistency, potentially leading to more robust and interpretable latent-dynamics models. In bridging topology, geometry, and learning theory, this work underscores the deep connections between modern representation learning and classical ideas from dynamical systems and manifold theory.

Finally, 
a potential application of these results is in feedback design. Suppose that we want to build a smooth feedback law that stabilizes an equilibrium in $M$. By moving to the latent space, we can map the problem into one of stabilization for systems on Euclidean spaces, for which rich techniques exist, see e.g.~\cite{sontag1998mathematical}.
We leave such extensions to future work.

\subsection*{Acknowledgments}

This material is based upon work supported by the Air Force Office of Scientific Research under award number FA9550-24-1-0299 (MK) 
and 
the Office of Naval Research under award number N00014-21-1-2431 (EDS).
We thank Jeremy Jordan for alerting us to relevant statements made by
\cite{Goodfellow-et-al-2016}.

\newpage
\clearpage
\appendix
\section{Pseudocode}

\begin{algorithm}[H]
\caption{Generate trajectories on \(S^1\)}
\begin{algorithmic}[1]
\State \textbf{Inputs:} \(N{=}512\), \(T{=}96\), \(\Delta t{=}0.04\)
\For{$i=1$ to $N$}
  \State Sample $\theta_{i,0}\sim\mathrm{Unif}[0,2\pi)$
  \For{$t=0$ to $T{-}2$}
    \State $\theta_{i,t+1}\gets \theta_{i,t}+\Delta t\cdot \sin(2\theta_{i,t})$
  \EndFor
\EndFor
\State $X_{i,t}\gets(\cos\theta_{i,t},\sin\theta_{i,t})$
\State \Return $X$, $X_t=X[:,0{:}T{-}1]$, $X_{t+1}=X[:,1{:}T]$
\end{algorithmic}
\end{algorithm}

\begin{algorithm}[H]
\caption{Autoencoder pretraining}
\begin{algorithmic}[1]
\State Initialize $E,D,h$
\For{epoch $=1$ to $500$}
  \State Sample minibatch $x$ from $X$
  \State $\phi\gets E(x)$; $\hat{x}\gets D(\phi)$
  \State $L_{\rm rec}\gets \|\,\hat{x}-x\,\|_2^2$
  \State Update $\{E,D,h\}$ by AdamW on $15.0\cdot L_{\rm rec}$
\EndFor
\end{algorithmic}
\end{algorithm}

\begin{algorithm}[H]
\caption{Annealed dynamics training}
\begin{algorithmic}[1]
\For{phase in \{2,3,4\}}
  \State Set $(W_{\rm rec},W_{\rm conj},W_{\rm lat1}, lr)$ per schedule
  \For{epoch in phase-epochs}
    \State Sample minibatch $x$; compute $\phi=E(x)$, $\hat{x}=D(\phi)$
    \State $L_{\rm rec}=\|\hat{x}-x\|_2^2$
    \State Sample minibatch $x^c$; compute $L_{\rm conj}$ via $J_D(\cdot)h(\cdot)$ vs $f(D(\cdot))$
    \State Sample pairs $(x_t,x_{t+1})$; compute $L_{\rm lat1}$ with RK4
    \State $L = W_{\rm rec}L_{\rm rec}+W_{\rm conj}L_{\rm conj}+W_{\rm lat1}L_{\rm lat1}$
    \State Update $\{E,D,h\}$ by AdamW on $L$
  \EndFor
\EndFor
\end{algorithmic}
\end{algorithm}

\begin{algorithm}[H]
\caption{Rollout from $x_0$}
\begin{algorithmic}[1]
\State $\phi_0\gets E(x_0)$
\For{$t=0$ to $T{-}1$}
  \State $\phi_{t+1}\gets \mathrm{RK4}(\phi_t,h;\Delta t)$
  \State $x_{t+1}\gets D(\phi_{t+1})$;\quad $\hat{\theta}_{t+1}\gets \arg(x_{t+1})$
\EndFor
\end{algorithmic}
\end{algorithm}

\begin{algorithm}[H]
\caption{Compute pulled-back vector field}
\begin{algorithmic}[1]
\State Uniform grid $\theta_k$, $k=1,\dots,K$ (e.g., $K=720$)
\State $\phi_k\gets E(\cos\theta_k,\sin\theta_k)$
\State Approx.\ $\dv{\phi}{\theta}(\theta_k)$ by centered differences in $k$
\State $h_k\gets h(\phi_k)$
\State $\widehat{\dot{\theta}}(\theta_k)\gets h_k \big/ \left(\dv{\phi}{\theta}(\theta_k)\right)$
\end{algorithmic}
\end{algorithm}

\section{Reproducibility}

Implementation uses PyTorch (float32), NumPy, Matplotlib. Random seeds
are fixed. The code trains on CPU (AdamW, weight decay
\(10^{-5}\)). Hyperparameters, schedules, and integrator are specified
above. All figures arise directly from the provided script without
manual post-processing. 

Code available from:

{\tt https://github.com/sontaglab/autoencode\_dynamics\_example.git}

\newpage
\bibliographystyle{amsalpha}
\bibliography{ref}

\end{document}